\documentclass[runningheads]{llncs}
\usepackage{graphicx}
\usepackage{enumerate}

\usepackage{wrapfig}
\usepackage[colorlinks=true,urlcolor=blue,linkcolor=black]{hyperref}
\usepackage{color}
\usepackage[english]{babel}
\usepackage{forloop}
\usepackage{eucal}
\usepackage{blindtext}
\usepackage{enumitem}
\usepackage{mathtools}
\usepackage{amssymb}
\usepackage{cite}
\usepackage{complexity}

\usepackage{caption}
\usepackage{subcaption}
\DeclareFontFamily{OT1}{pzc}{}
\DeclareFontShape{OT1}{pzc}{m}{it}{<-> s * [1.100] pzcmi7t}{}
\DeclareMathAlphabet{\mathpzc}{OT1}{pzc}{m}{it}

\usepackage[ruled,linesnumbered]{algorithm2e}
\SetKwInput{KwInit}{Initialisation}
\SetKwBlock{Repeat}{repeat}{}
\SetKw{Continue}{continue}
\SetKw{Break}{break}
\SetKw{CondAnd}{and}
\SetKw{Fail}{fail}
\SetKw{Yield}{yield}
\SetKw{YieldFailure}{yield failure}
\SetKwProg{Fn}{function}{:}{end}
\DontPrintSemicolon

\usepackage{cleveref}

\crefname{lemma}{lemma}{lemmas}
\Crefname{lemma}{Lemma}{Lemmas}
\crefname{theorem}{theorem}{theorems}
\Crefname{theorem}{Theorem}{Theorems}
\crefname{definition}{definition}{definitions}
\Crefname{definition}{Definition}{Definitions}
\crefname{remark}{remark}{remarks}
\Crefname{remark}{Remark}{Remarks}

\begin{document}
\title{Deterministic Graph-Walking Program Mining}
\author{Peter Belcak \and Roger Wattenhofer}
\authorrunning{P. Belcak and R. Wattenhofer}
\institute{ETH Zürich, Rämistrasse 101, 8092 Zürich \\ \email{\{belcak,wattenhofer\}@ethz.ch}}
\maketitle
\begin{abstract}

Owing to their versatility, graph structures admit representations of intricate relationships between the separate entities comprising the data.
We formalise the notion of connection between two vertex sets in terms of edge and vertex features by introducing graph-walking programs.
We give two algorithms for mining of deterministic graph-walking programs that yield programs in the order of increasing length.
These programs characterise linear long-distance relationships between the given two vertex sets in the context of the whole graph.

\end{abstract}
\section{Introduction}
\label{section:introduction}

While data has been stored in the form of tables since time immemorial, more complex data is often represented with graphs.
This is because graph databases generalise conventional table-driven data storage methods, allowing for modelling of involved relationships among entities represented therein.
As such, graph analysis and mining methods will be at the center of attention when it comes to contextual understanding of relationships between individual datapoints within a large database.

Here we investigate the identification of one type of such relationship between two groups of graph's vertices.
For an illustrative example (\Cref{figure:prosecution_example}), consider an individual who has just graduated from high school (starting qualifications $S$) and aims to reach a target career (target qualifications $T$ ) while being permitted only one study focus at the time -- e.g. studying either social or biological sciences, but not both.
What sequence of decisions with regards to their study foci should they take?
Notice that the choice of focus made at every stage of the individual's education leads to a restriction on what qualifications they can obtain in the future.
Thus, at each stage of their education, they will need to focus on qualifications that are pre-requisites for those that lead to $T$.
Attempting to solve the problem on our own, we can search for a sequence of instructions that leads us from $S$ to $T$ either by naively tracing out paths from $S$ and reading out instructions one by one, or by enumerating all possible instruction sequences and then verifying if they indeed do lead from $S$ to $T$.

\begin{figure}[h!]
    \centering
    \includegraphics[width=\textwidth]{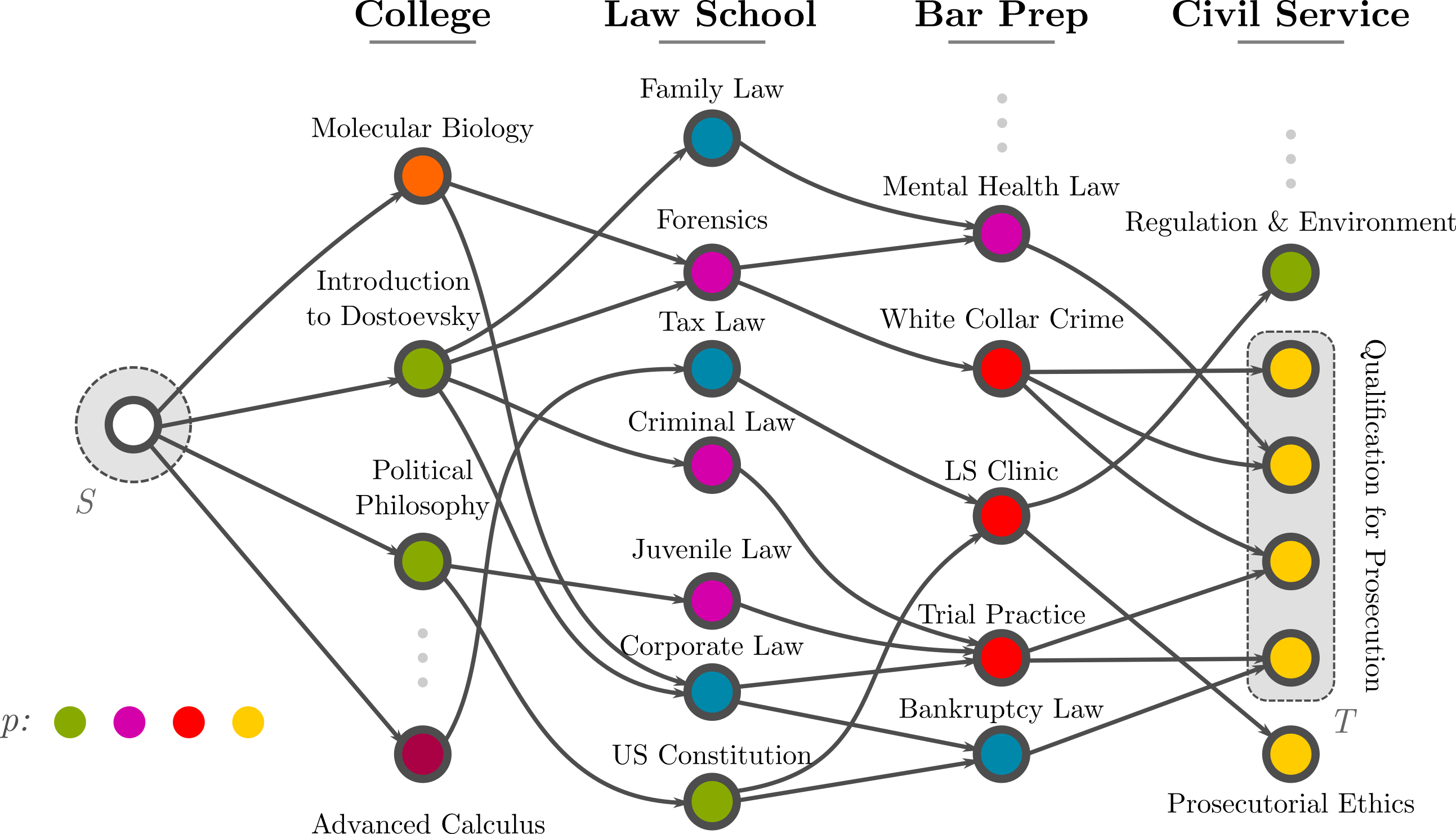}
    \caption{
        A depiction of the qualification pathway $p$ for an individual who has just graduated from high school (the starting qualification in $S$) and aims to become a prosecutor in the United States (target qualifications $T$).
        Vertices, colours, edges represent qualifications, qualification foci, and dependencies, respectively.
        Going into college, they will need to choose a focus that maximises their chances of being admitted to a law school (most likely social sciences).
        In law school, they will need to focus on criminal rather than tax or corporate law and prepare diligently for their barrister examinations.
        Finally, they will need to satisfy the necessary pre-requisites of civil service before becoming a prosecutor.
    }
    \label{figure:prosecution_example}
\end{figure}

\begin{figure}[h!]
    \centering
    \includegraphics[width=\textwidth]{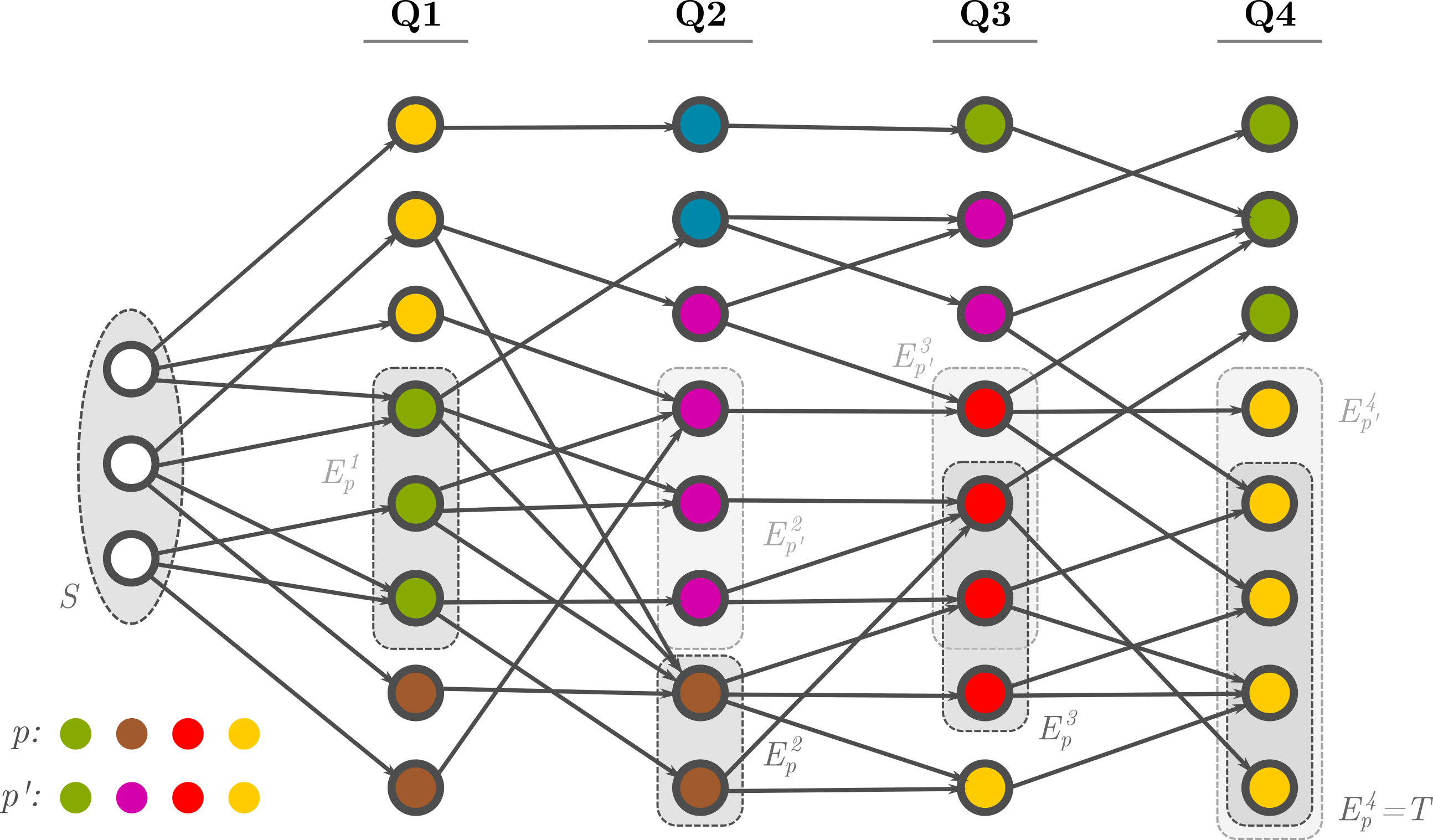}
    \caption{
        An example of the qualification program problem with multiple starting qualifications.
        An individual possessing three different qualifications (vertices in $S$) of the same focus/type (colour) seeks to attain any of the qualifications in $T$ such that they always do qualifications of the same type.
        Each of their qualifications, however, is a pre-requisite (directed edge) for a slightly different set of later qualifications, and it will take at least four steps to reach $T$ from any qualification in $S$.
        $p$ gives a program in which they first work towards green, then brown, then red, and finally yellow qualifications, and exactly the qualifications in $T$ are achieved.
        $p'$ gives a program \textit{green-purple-red-yellow}, in which there is some overlap with $T$ but an additional yellow qualification not in $T$ is achieved.
    }
    \label{figure:qualification_example}
\end{figure}

\begin{figure}[h!]
     \centering
     \begin{subfigure}{0.45\textwidth}
         \centering
         \includegraphics[width=\textwidth]{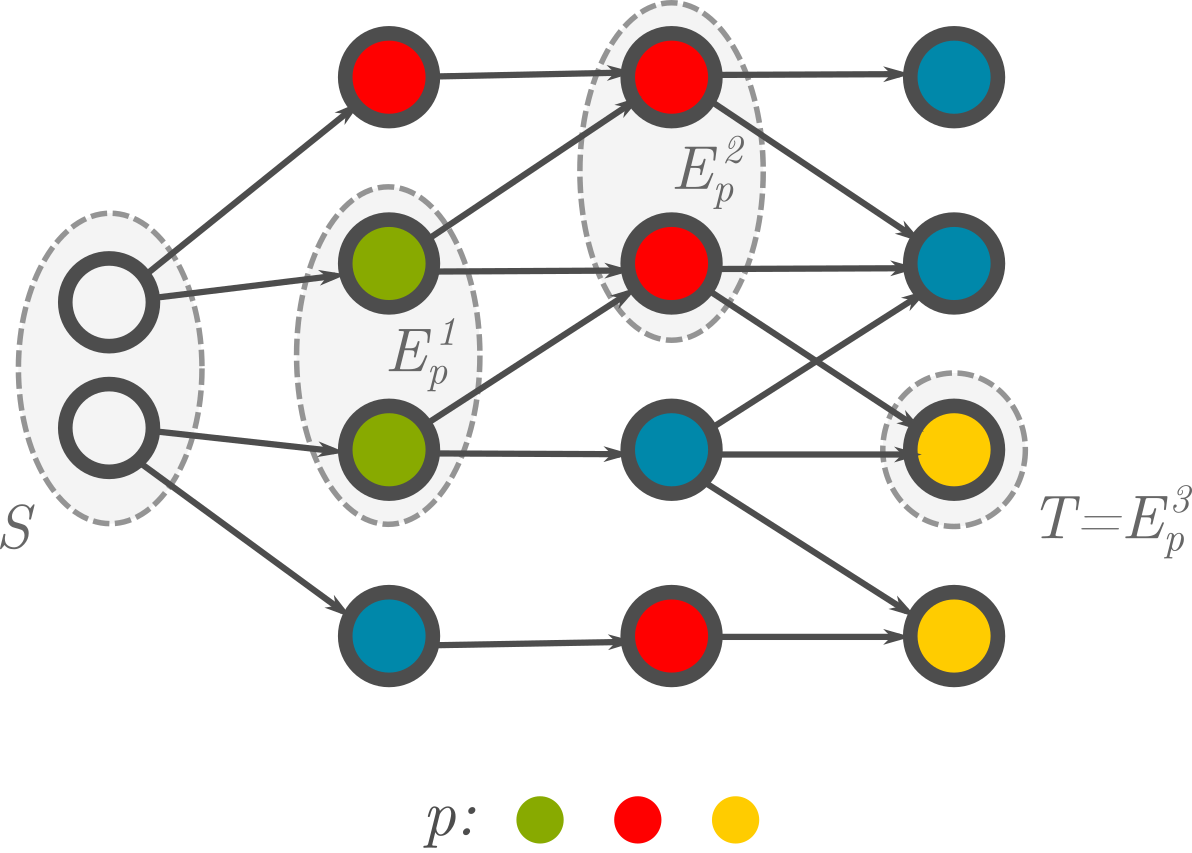}
         \vfill
     \end{subfigure}
     \hfill
     \begin{subfigure}{0.45\textwidth}
         \centering
         \includegraphics[width=\textwidth]{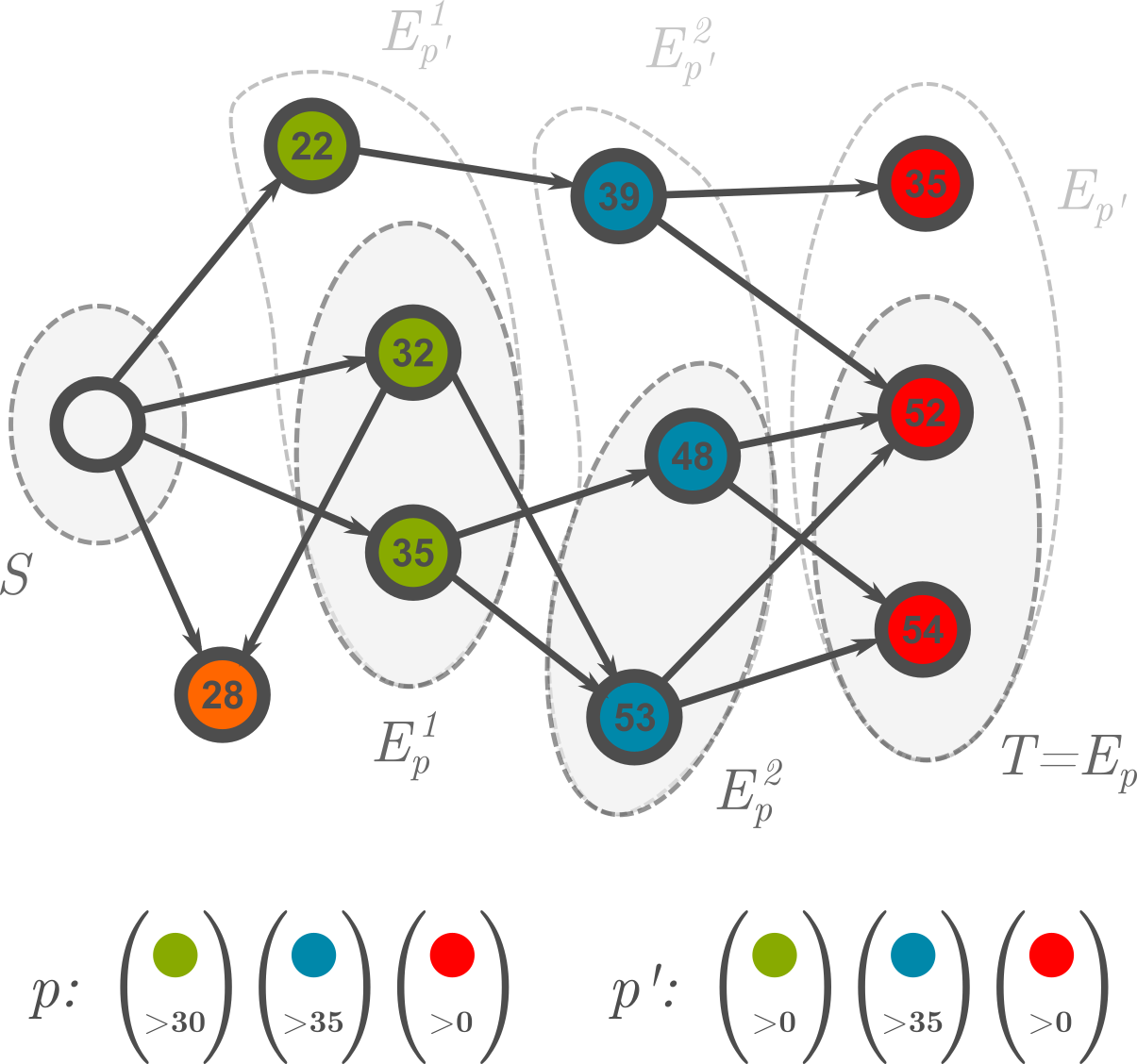}
     \end{subfigure}
    \caption{
        Two illustrative examples with solutions.
        Recall the objective to find a sequence of conditions on features that leads imagined agent from the vertices of $S$ to the vertices of $T$.
        \textit{Left.} The simple colour program $p$ can be used to instruct the agent starting at vertices of $S$ to proceed towards $T$. On the first step $E_p^1$ is reached.
        On the second step, the agent proceeds to the red nodes marked by $E_p^2$.
        On the third step, the agent proceeds to $T$.
        \textit{green-blue-yellow} would be another feasible program.
        \textit{Right.} Two simple toset (totally-ordered set) programs $p,p'$ are presented, conditioning on two feature dimensions of a general, unlayered graph: colour and integers.
        Agent starting in $S$ reaches exactly $T$ if following $p$ ($p$ is feasible and exact) but ends up at a strict superset $E_{p'}$ of $T$ if she follows $p'$ (hence $p'$ is infeasible).
    }
    \label{figure:simple_examples}
\end{figure}

Informally, let $G$ be a graph where every vertex has some features, and let $S,T$ be vertex sets.
We ask the following question: ``How is $S$ connected to $T$ in terms of the features of the vertices between them?''
Or, alternatively: ``What instructions should agents starting at the vertices of $S$ follow in order to reach $T$''?

Revisiting the example above, if $G$ is a map of qualifications, with $G$ being qualifications and directed edge $(v_1, v_2)$ denoting that $v_1$ is a pre-requisite qualification for obtaining $v_2$, one could iteratively ask ``what type of qualifications from among the qualifications I am eligible for now should I achieve to eventually reach my target qualifications $T$''?
A good answer would give a sequential list of characteristics of qualifications.
Of course, getting all possible qualifications at every stage would likely lead to obtaining qualifications in $T$, but ideally one would not be doing more than absolutely necessary.
We illustrate a variant of this example in which each qualification is only characterised by its type in \Cref{figure:qualification_example}.

We investigate this problem and aim to give answers in terms of lists of instructions for a single graph-walking agent that can be present at multiple vertices at the same time.
We dub such lists of instructions \textit{simple graph-walking programs}.

Let $G = (V, E)$ be a directed multi-graph, $\mathcal{F}_1,\dotsc, \mathcal{F}_z$ spaces of the features appearing in $G$, and $\mathcal{F} := \prod_{i} \mathcal{F}_i  \cup \{ \emptyset_i \}$ their product, where $\emptyset_i$ denotes that a given graph element is not assigned feature $i$.
Let $\phi_V: V \to \mathcal{F}, \phi_E: E \to \mathcal{F}$ be the vertex and edge feature mappings.
Let $p:c_1\cdots c_{2n}$ be a simple graph-walking program -- a list of vertex movement selection instructions, i.e. functions on $c_t: \mathcal{F} \to \left\{0,1\right\}$.
Consider an agent, located at $E_p^t$ for any time $t \geq 0$, that begins at the set of vertices $S = E_p^0$ and at each time $t \geq 1$ decides to proceed only to those vertices $v$ in out-neighbourhood of $E_p^{t-1}$ connected by edges $e$ whose feature vectors $\phi_V(v),\phi_E(e)$ satisfy $c_{2t-1}(\phi_E(e)) = 1 = c_{2t}(\phi_V(v))$.
The problem of \textit{simple graph-walking program mining} is the problem of finding lists of instructions $p$ such that an agent following it reaches $T$ by the end of the program ($\emptyset \neq E_p^n \subseteq T$).
See \Cref{figure:simple_examples}.
Alternatively and more in line with the program synthesis literature, we can speak of \textit{graph-walking program induction} or \textit{synthesis}, with the triple $(G,S,T)$ forming the inputs for the induction of the programs.
We call programs that satisfy $\emptyset \neq E_p^n \subseteq T$ \textit{feasible}, and programs that achieve the equality \textit{exact}.

The resulting programs are not frequent patterns, nor do they characterise the graphs locally; they characterize long-distance relationships between groups of vertices in $G$ in terms of $\mathcal{F}$.
Nevertheless, for a given pair of vertex sets $S,T$ there are often many feasible programs, and our algorithms carry some characteristics of a priori graph pattern mining.
We talk of ``simple'' programs because there are more elaborate program structures that could be studied in this setting (such as those that posses memory), and of ``deterministic'' programs since the instruction/criterion $c_i$ always gives either $1$ or $0$ as firm directions to the agent walking the graph.

The difficulty of this problem dwells in it being a cunning composition of two necessary sub-tasks -- \textit{path search} and \textit{classification} -- well-understood and studied in graph theory and machine learning, respectively.
This is because it is not enough to find a possible program walk from a vertex in $S$ to a vertex in $T$ -- one has to choose from all possible walks for all possible choices of the pair $(s,t) \in S \times T$, and then find a subset of these walks for which a single graph-walking program can be used.
In other words, it is necessary to both \textit{discover} possible walks, and \textit{discriminate} among them.

The mining algorithms we propose are \textit{correct} (they return only valid programs) and \textit{complete} (proceeding in stages, they always yield all valid programs up to some length $\ell$ before looking for longer programs).

This paper reviews related work (\Cref{section:relatedWork}), describes the problem of simple colour program mining, gives an algorithm for the task (\Cref{section:simpleColourPrograms}) -- which, to the best of our knowledge, has never been addressed in the literature before -- and extends simple colour program mining to simple totally-ordered-set programs (\Cref{section:simpleTosetPrograms}).

\section{Related Work}
\label{section:relatedWork}

Our effort lies at the intersection of two areas, namely graph program synthesis and analysis of complex networks.

Algorithmic program synthesis \cite{bodik2013algorithmic}, traditionally considered a problem in deductive theorem proving, has recently been looked at as a search problem with constraints such as a logical specification of the program behaviour \cite{feng2018program}, syntactic template \cite{alur2018search, desai2016program}, and, most recently, previously discovered program fragments and utility functions \cite{huang2021neural, ellis2021dreamcoder}.
Several new methods combine enumerative search with deduction, aiming to rule out infeasible sub-programs as soon as possible \cite{feng2017component, feser2015synthesizing}.
While relevant to us in their intent, the methods are domain-specific and do not extend to programs on graphs.

Under the paradigm of program search within a restricted graph context, Yaghmazadeh et al. \cite{yaghmazadeh2016synthesizing} study the synthesis of transformations on tree-structured data and employ a combination of SMT solving and decision tree learning.
Their synthesis system, \textsc{Hades}, outputs programs in a custom domain-specific language for tree-transformation.
Their approach considers entire graphs at the same time and while it does provide insights into construction of programs for graphs, it does not extend to the graph walk scenarios.

Wang et al. \cite{wang2017synthesizing} give a synthesis algorithm for queries $q$ on a set of tables $T_1,T_2,\dotsc,T_k$ that output records of a target table $T_\text{out}$.
Their approach to query-walk search is syntactic (in contrast to treating the database schema as a graph) and relies on simple enumeration of possible table visits, something our algorithms avoid with further constraints on the search space.

Mendelzon and Wood \cite{mendelzon1995finding} consider the problem of finding pairs of vertices in a graph connected by simple paths such that the trace of the labels of the vertices traversed satisfies a given regular expression.
While being perhaps closest to our work, their goal is to find paths that satisfy a constraint, rather than finding constraints for which connecting paths exist, thus having one degree of freedom fewer.

A sub-branch of graph pattern mining considers the special case of mining frequent paths  \cite{guha2009efficiently,gudes2005mining}.
We note that while the knowledge of frequent paths in a graph might potentially accelerate the search for solutions for the graph-walking program mining problems, methods for frequent path mining are of little direct use since we seek programs that go between particular sets $S,T$.

Finally, the literature on analysis of complex networks frequently focuses on characterisation of elements of networks in terms of their interactions with their neighbourhoods.
Among the examined characterisations are the notions of structural equivalence \cite{lorrain1971structural, breiger1975algorithm, burt1976positions}, regular equivalence \cite{white1983graph}, or other partitioning strategies \cite{winship1983roles}.
Further, random walks of graphs are frequently employed to help with analysis of graphs as whole \cite{page1999pagerank, cooper2016fast} or as sum of its communities \cite{andersen2006local, avrachenkov2013choice}, but to our knowledge, no work so far has investigated the identification of relationships between nodes related from beyond close neighbourhoods. 

\section{Simple Colour Programs}\label{section:simpleColourPrograms}

Without any loss of generality we restrict ourselves to program mining on simple directed graphs with featureless edges, to whom directed multi-graphs with edge features can be converted by replacing every edge with a vertex inheriting edge's features and retaining its endpoints as the only neighbours.

We focus on \textit{simple colour programs} -- programs $p: c_1 \cdots c_n$ such that $c_i$ is the colour of the out-neighbours of the vertices reached by the prefix program $c_1\cdots c_{i-1}$ whom the agent executing $p$ should proceed to at step $i$.
The program instructions (criteria) are thus colours.

\subsection{Preliminaries}
Let $G$ be a simple directed $k$-coloured graph. The colouring does not have to be proper.
Let $\emptyset \neq S, T \subseteq V$.
Denote by $\mathpzc{c}(v)$ the colour of vertex $v$, by $\mathpzc{c}(A)$ the set of colours of the vertices in $A \subseteq V$, and by $C_c(A)$ the set of vertices of $A$ with colour $c$.
Call set of vertices $A$ monochromatic if $\mathpzc{c}(A)$ is a singleton set.
Denote the out- and in-neighbours of $A$ by $N_o(A)$ and $N_i(A)$ respectively.
Shorten $n$ applications of $N$, i.e. $N_o(N...N(A)...))$, to $N^n_o(A)$, and similarly for $N_i$.
For convenience, define $N_o^0(A) := A =: N_i^0(A)$.

\begin{definition}[Simple Colour Program]
    $p:c_1 c_2 \cdots c_n$ is a simple colour program (SCP) of length $n$ iff $1 \leq c_i \leq k$ for all $1 \leq i \leq n$.
\end{definition}

Use $\epsilon$ for empty program -- the unique program of length $0$, $p_i$ for $c_i$, $p_{\leq i}$ for the prefix $c_1 \cdots c_i$ of $p$ for $1 \leq i \leq n$, and $p_{\geq i}$ for the suffix $c_i \cdots c_n$ of $p$.

\begin{definition}[Program Endpoints]
    For $p:c_1 c_2 \cdots c_n$ an SCP, define $E^{i}_p(S)$ for $0 \leq i \leq n$ recursively as follows:
    \begin{enumerate}
        \item $E^{0}_p(S) = S$,
        \item For $i>0$, $E^i_p(S) = C_{c_i}(N_o(E^{i-1}_p(S)))$,
    \end{enumerate}
    and denote $E^n_p(S)$ by $E_p(S)$.
\end{definition}

\begin{definition}[Feasible and Exact SCP]
    $p$ is feasible iff $\emptyset \neq E_p(S) \subseteq T$\label{definition:feasibleSCP}, and exact iff $E_p(S) = T$.
\end{definition}

\begin{definition}[Partial Halting]
    $p$ partially halts (on $G$) if there exists an $0 \leq i < n$ and $v$ such that $v \in E^i_p(S)$ but $c_{i+1} \not \in \mathpzc{c}(N_o(v))$. 
\end{definition}

In other words, $p$ partially halts if it ever reaches a vertex from which it is impossible to proceed while still following $p$.

\begin{lemma}
    If $p$ is a feasible program that does not partially halt, then for all $0 \leq i < n$ there exists a colour $c_{\text{to}}$ such that $\emptyset \neq C_{c_{\text{to}}}(N_o(E_p^i(S))) \subseteq N^{n-i+1}_i(T)$.
\end{lemma}

\begin{proof}
    For each $i$, take $c_{to} = p_i$.
\end{proof}

\begin{definition}[Complete Halting]
    $p$ halts (completely) on $G$ if there exists an $1 \leq i < n$ such that $c_{i+1} \not \in \mathpzc{c}(N_o(E^i_p(S)))$.
    Equivalently, there is an $1 \leq i < n$ such that $E^{i+1}_p(S) = \emptyset$.
\end{definition}


\begin{lemma}
    Assume that a feasible SCP exists for $S,T$. Then \label{lemma:scpExistenceNecessaryConditions}
    \begin{enumerate}
        \item There exists a walk $w$ from $s \in S$ to $t \in T$ such that the colours of the vertices from $s$ to $t$ give a feasible program for $S,T$.\label{item:walkGivesProgram}
        \item If a program that does not partially halt exists, for every $s \in S$ there are $t \in T, w$ as in \cref{item:walkGivesProgram}.
        \item If an exact program exists, for every $t \in T$ there are $s \in S, w$ as in \cref{item:walkGivesProgram}.
    \end{enumerate}
\end{lemma}

\begin{proof}
    Let $p:c_1\cdots c_n$ be a feasible program.
   \begin{enumerate}
       \item Take any $w_n \in E_p(S)$. If $n=1$ we are done. If not, prepend it by any $w_{n-1} \in N_i(w_n) \cap E_p^{n-1}(S)$ which is non-empty as $p$ is feasible and therefore does not halt, and observe that $\mathpzc{c}(w_{n-1}) = p_{n-1}$.
       Repeat this process for a total of $n$ times. Then $w_{0} \in N_i(w_1) \cap E_p^{0}(S) \subseteq S$ and $w_1 \cdots w_n$ is a walk from a vertex in $S$ to a vertex in $T$ such that the colours of the vertices it visits give precisely the program $p$.
       \item If $p$ does not partially halt then for every $s \in S$, $E_p^i(s) \neq \emptyset$ and $E_p(s) \subseteq T$.
       So $p$ is a feasible $\{s\},T$-program, and hence item 1 applies.
       \item If $p$ is exact, $E_p(S) = T$ and the proof of item 1 also gives this stronger statement.
   \end{enumerate}
\end{proof}

\begin{definition}[Cover]
    For $A,B \subseteq V$ we say that the vertices $A$ cover $B$ by $c$ iff $C_c(N_o(A)) \supseteq B$.
\end{definition}

\begin{definition}[Injection]\label{definition:injection}
    For $\emptyset \neq A,B \subseteq V$ we say that the vertices $A$ inject $B$ by $c$ iff $\emptyset \neq C_c(N_o(A)) \subseteq B$.
    If that is the case, we call $A$ a $c$-injection into $B$.
\end{definition}

\begin{definition}[Spanning]\label{definition:spanning}
    We say that $A$ outspans $B$ by $c$ iff $C_c(N_o(A)) \backslash B \neq \emptyset$, and that $A$ spans $B$ by $c$ iff $A$ covers $B$ by $c$ but does not outspan $B$ by $c$.
\end{definition}

Notice that $A$ $c$-injects $B$ iff $A$ does not $c$-outspan $B$ and $A$ is not a $c$-halting point.

\begin{lemma}[Cover-Inject Behaviour of Intermediate Endpoints]\label{lemma:coveringButNotOutspanningInIntermediateEndpoints}
    For any program $p$ decomposed as $\pi c d \sigma$, $E_{\pi}(S)$ spans $E_{\pi c}(S)$ by $c$ but does not outspan $C_c(N_i(E_{\pi c d}(S)))$.
\end{lemma}

\begin{proof}
    Let $p:\pi c d \sigma$ be a feasible program.
    \\
    Since $E_{\pi c}(S) = C_c(N_o(E_{\pi}(S)))$, $E_{\pi}$ spans $E_{\pi c}$ by $c$.
    \\
    Further, since $E_{\pi c d}(S) \subseteq N_o(E_{\pi c}(S))$, $C_c(N_o(E_{\pi}(S))) \subseteq C_c(N_i(E_{\pi c d}(S)))$, so $E_{\pi}$ does not outspan $C_c(N_i(E_{\pi c d}(S)))$.
\end{proof}

\begin{proposition}
    \label{proposition:npnessOfSCPI}
    The problems of finding a feasible simple colour program and exact simple colour program are $\NP$.
\end{proposition}

\begin{proof}
    Let $G,S,T$ and a candidate simple colour program $p$ be inputs.

    To verify the the certificate $p$ one can simulate the actions of a set graph-walking agent.
    Starting at $S = E_p^0(S)$, the agent visits vertices $N_o(E_p^0(S)) \subseteq V$ and compare their colour to $c_1$.
    Searching for edges originating at a vertex, searching for vertex colour, and comparing vertex colours to $c_i$ is a polynomial-time operation.
    There are always at most $\left|V\right|$ vertices whose out-neighbours must be visited, and this operation is repeated for $1 \leq i \leq n$.
    Hence the verification of the certificate is a polynomial-time operation.
\end{proof}

\subsection{Viable Injection Basis Enumeration}
We present a simple colour program mining algorithm
constructing candidate programs from space of possibilities reduced by considering only those injections that cover ``enough'' vertices to have hope of reaching $T$.
This is captured by the notion of \textit{pseudo-basis}.

\begin{definition}[Basis]\label{definition:basis}
    We say that $\mathcal{B}$ is a $c$-basis for $B$ iff $\mathcal{B}$ spans $B$ by $c$ and $\mathcal{B}$ is a minimal such set, i.e. for any $v \in \mathcal{B}$, $\mathcal{B}-v$ does not span $B$ by $c$.
\end{definition}

\begin{lemma}\label{lemma:reductionLemma}
    $A$ is a $c$-spanning set for $B$ iff it contains a $c$-basis for $B$ and does not $c$-outspan $B$.
\end{lemma}

\begin{proof}
    Let $A$ be a $c$-spanning set for $B$.
    Then it has a basis (remove vertices until none can be removed without making it a non-spanning set) and by definition does not outspan $B$.
    
    Conversely, let $A$ be a set that contains a $c$-basis $\mathcal{B}$ but does not outspan $B$.
    Since $B = C_c(N_o(\mathcal{B}) \subseteq C_c(N_o(A))$, $A$ covers $B$.
    Hence $A$ $c$-spans $B$.
\end{proof}

\begin{definition}[$c$-Pseudo-Basis]\label{definition:cpseudoBasis}
    We say that $\mathcal{B}$ is a $c$-pseudo-basis for $(B,M)$ iff $\mathcal{B}$ $c$-covers $B$, $\mathcal{B}$ $c$-injects $M$, and $\mathcal{B}$ is a minimal such set, i.e. for any $v \in \mathcal{B}$, $\mathcal{B}-v$ does not $c$-cover $B$.
\end{definition}

\begin{remark}
    Notice $\mathcal{B}$ is a $c$-basis for $B$ if it is a $c$-pseudo basis for $(B,B)$.
\end{remark}

The utility of pseudo-bases comes from being a type of injection into $M$ that covers all vertices of the out-neighbourhood designated as essential ($B$).
This allows us to remove from our search space those injections that do not cover sufficiently many of its out-neighbours to fully reach $T$.
More specifically, our strategy is to
\begin{enumerate}
    \item consider all bases $\mathcal{B}_\bullet^1$ for $T$,
    \item consider all sets $\mathcal{B}_\bullet^2$ covering each $\mathcal{B}_\bullet^1$ but not outspanning the corresponding monochromatic in-neighbourhoods $M_\bullet^1$ of $T$, i.e. the pseudo-bases for each $(\mathcal{B}_\bullet^1, M_\bullet^1)$,
    \item do the same for pairings $(\mathcal{B}_\bullet^i, M_\bullet^i)$, $i \geq 2$.
    If the candidate pseudo-basis $\mathcal{B}_j^i$ lies in $S$ and $S$ is in turn fully contained in the in-neighbourhood of the appropriate $c$-span of $\mathcal{B}_j^i$, a valid program has been found.
\end{enumerate}

\noindent
In \Cref{alg:vibe}, let $Q$ and $Q_{\text{next}}$ be queues of triples drawn from \textit{programs $\times$ vertex-sets $\times$ monochromatic vertex-sets}, $P$ be a set of \textit{programs}.
Each triple $(p, B, M)$ in $Q,Q_{\text{next}}$ represents a candidate program $p$, from where to begin $B$ in order to reach $T$, and a monochromatic in-neighbourhood $M \subseteq N_i(E_{p_{\leq i}}(B))$.

\begin{algorithm}[h!]
\caption{(VIBE) Viable Injection Basis Enumeration}\label{alg:vibe}
\KwInit{$Q_{\text{next}}$ contains only $(\epsilon, T, T)$, $Q$ and $P$ are empty}

\ForEach{$\ell \geq 0$ such that $T \subseteq N_o^{\ell}(S)$\label{line:vibeForwardPass}}{
    $Q \gets Q_{\text{next}}$\;
    empty $Q_{\text{next}}$\;
    \While{$Q$ is not empty\label{line:vibeBackwardPass}}{
        pop $(p, B, M)$ from $Q$\;
        $n \gets $ length of $p$\;
        \;
        \If{$n=\ell$}{
            \If{$B \subseteq S \subseteq N_i(E_{p_{\leq 1}}(B))$\label{line:vibeIfIsFeasible}}{
                add $p$ to $P$\label{line:vibeAddFeasibleProgram}\;
            }
            push $(p, B, M)$ into $Q_{\text{next}}$\;
            \Continue
        }
        \;
        \ForEach{$c \in \mathpzc{c}\left( B \right)$\label{line:vibeSearchAllValidExtensionsStart}\label{line:vibeColoursToConsider}}{
            $N \gets N_o^{\ell-n-1}(S) \cap N_i(B)$\label{line:declareNeighboursToConsiderFromForwardPass}\;
            \ForEach{$d \in \mathpzc{c}(N)$\label{line:vibeSearchAllValidMonochromaticPresetsStart}}{
                \eIf{$\ell \neq n+1$\label{line:vibeDeclareColouredNeighbourhood}}{
                    $N_d \gets C_d(N)$\;
                }{
                    $N_d \gets N$\;
                }
                \ForEach{$c$-pseudo-basis $\mathcal{B} \subseteq N_d$ for $(B, M)$\label{line:vibeSearchAllValidInjectionsStart}}{
                    push $(cp, \mathcal{B}, N_d)$ into $Q$\label{line:vibeAddExtendedPossibleProgram}
                }\label{line:vibeSearchAllValidMonochromaticPresetsEnd}
            }\label{line:vibeSearchAllValidInjectionsEnd}
        }\label{line:vibeSearchAllValidExtensionsEnd}
    }\label{line:vibeCombinatorialLoopEnd}
}
\end{algorithm}

\begin{proposition}
    The following hold.\label{proposition:characteristicsOfVIBE}
    \begin{enumerate}
        \item \Cref{alg:vibe} is correct in the sense that all programs in $P$ are exact programs for $S,T$.
        \item \Cref{alg:vibe} is complete in the sense that whenever execution exists the loop closing at \cref{line:vibeCombinatorialLoopEnd}, $P$ contains all exact programs for $S,T$ of length $\ell$.
    \end{enumerate}
\end{proposition}

\begin{proof}
    Observe that for every $(p,B,M)$ in $Q$, $B \subseteq M$, $M$ is a monochromatic set, and the spans of $B$ and $M$ are the same.
    \begin{enumerate}
        \item First, we show inductively that every triple $(p,B,M)$ in $Q$ is such that for any set $B \subseteq A \subseteq M$, $E_p(A) = T$.
        
        The base case (stemming from the initialisation of $Q_{\text{next}}$) is straightforward as $E_\epsilon(T) = T$.
        Suppose $p \neq \epsilon$.
        Then there is a colour $c$ and a shorter program $q$ such that $p = cq$, as reaching \cref{line:vibeAddExtendedPossibleProgram} is the only way for a non-empty program to enter $Q$.
        
        So there is a triple $(q, B', M')$ and such that $B$ is a pseudo-basis for $(B', M')$ (cf. \cref{line:vibeSearchAllValidInjectionsStart}).
        Thus $B' \subseteq E_c(B) \subseteq M'$ by \Cref{definition:cpseudoBasis}.
        But $E_p(B) = E_{q}(E_c(B))$, so by the inductive hypothesis $E_p(B) = T$.
        
        Now, every program $p$ in $P$ must have been added on \cref{line:vibeAddFeasibleProgram}, so necessarily there is a triple $(p,B,M)$ that was once in $Q$ s.t. $B \subseteq S \subseteq N_i(E_{p_{\leq 1}}(B))$.
        Since $S \subseteq N_i(E_{p_{\leq 1}}(B))$ we have $E_{p_{\leq 1}}(S) = E_{p_{\leq 1}}(B)$, so
        $$
            E_p(S) = E_{p_{>1}}(E_{p_{\leq 1}}(S)) = E_{p_{>1}}(E_{p_{\leq 1}}(B)) = E_p(B) = T.
        $$
        
        \item
        Let $p$ be an exact program for $S,T$.
        
        Focusing on the combinatorial loop of \cref{line:vibeBackwardPass} and ignoring the caching by $Q_{\text{next}}$, we shall show inductively that for every prefix-suffix decomposition $\pi \sigma = p$ of p there is a triple $(\sigma, B, M)$ s.t. $B \subseteq E_\pi(S) \subseteq M$ that appears on $Q$.
        
        For the base case with suffix $\sigma = \epsilon$, $p$ is an exact program, so $T \subseteq E_p(S) \subseteq T$. This is the triple $(\epsilon, T, T)$ found in initialisation.
        
        Assume the inductive hypothesis holds for shorter suffixes and decompose $p$ to $\pi' c \sigma$.
        Since by the inductive hypothesis there is a triple $(\sigma, B, M)$ s.t. $B \subseteq E_{\pi'c}(S) \subseteq M$, $c \in \mathpzc{c}(B)$ and $E_{\pi'}(S) \subseteq N \neq \emptyset$ on \cref{line:declareNeighboursToConsiderFromForwardPass}.\
        Notice also that $E_{\pi'}$ is further monochromatic whenever $\pi' \neq \epsilon$, so by the branching on \cref{line:vibeDeclareColouredNeighbourhood} $E_{\pi'} \subseteq N_d$.
        Further, as a consequence of \Cref{lemma:coveringButNotOutspanningInIntermediateEndpoints}, $E_{\pi'}(S)$ covers $E_{\pi' c}$ but does not outspan $M$, so $E_{\pi'}(S)$ is a $c$-pseudo-basis for $(B, M)$, proving the inductive hypothesis.
        
        Now, whenever $\ell = n$ execution will reach \cref{line:vibeIfIsFeasible} with various triples $(p, B, M)$.
        Decompose $p = \epsilon p$.
        Then by the hereproven induction one of them will be such that $B \subseteq E_\epsilon(S) = S \subseteq M$, and by the aboveproven induction also $E_p(S) = T$.
        So $p$ will be added to $P$ on \cref{line:vibeAddFeasibleProgram}.
        This completes the proof of completeness.
    \end{enumerate}
\end{proof}

\Cref{alg:vibe} can be easily modified to also yield feasible programs.
This can be done by altering \cref{line:vibeSearchAllValidInjectionsStart} to give $c$-injections into $T$ if $n=0$, and execute the present behaviour otherwise.
Alternatively and equivalently, one can just pre-compute all viable $c$-injections, their monochromatic peers, and initialize $Q_{\text{next}}$ to their set in arbitrary order.

The completeness of \Cref{alg:vibe} combined with \Cref{lemma:coveringButNotOutspanningInIntermediateEndpoints} highlight the role of existence of appropriate pseudo-basis as a necessary and sufficient condition for local feasible program existence.

\section{Simple Toset Programs}
\label{section:simpleTosetPrograms}

Extending on algorithms of \Cref{section:simpleColourPrograms} we now consider a more general setting where there are multiple features at each vertex, and the feature spaces admit a total order.
See \Cref{figure:simple_examples}{-}\textit{Right.} for an example.

\begin{definition}[Criterion]
    \label{definition:criterion}
    Let $c$ be a triple $(f, \omega, \nu)$ where $f$ is a feature, $\omega$ is one of the operators $<,\leq,=,\geq, >$, and $\nu \in \mathcal{F}_f$. Then $c$ is an atomic criterion.
    Inductively, $c$ is a criterion if it is either an atomic criterion, a conjunction of criteria, or a disjunction of criteria.
\end{definition}

\begin{definition}[Simple Toset Program]
    We say that $p: c_1 c_2 \cdots c_n$ is a Simple Toset Program (STP) if each $c_i$ is a criterion.
\end{definition}

\begin{definition}[Criterion Satisfaction]
    We say that $v \in V$ satisfies the atomic criterion $c = (f, \omega, \nu)$ iff $\phi_V(v) \omega \nu$. We then re-define $C_c(A)$ in the context of STPs to mean the set of all vertices in $A$ that satisfy $c$.
    If $c$ is a criterion, we say that $v \in V$ satisfies $c$ iff
    \begin{itemize}
        \item $c$ is an atomic criterion and $v$ satisfies $c$ in the sense for atomic criteria, or
        \item $c$ is a disjunction of criteria $c_1 \lor \dots \lor c_k$ and $v$ satisfies at least one of $c_1,\dotsc,c_k$, or
        \item $c$ is a conjunction of criteria $c_1 \land \dots \land c_k$ and $v$ satisfies all of $c_1,\dotsc,c_k$.
    \end{itemize}
\end{definition}

All of the previous notions such as endpoints $E_p(\cdot)$, program feasibility, or program exactness, can be readily carried over from SCPs to STPs.

\begin{proposition}
    \label{proposition:npnessOfSTPI}
    The problems of finding a feasible simple toset program and exact simple toset program are $\NP$.
\end{proposition}

\begin{proof}
    See the proof of \Cref{proposition:npnessOfSCPI}, with the difference that instead of comparing colours we verify whether a criterion (cf. \Cref{definition:criterion}) is satisfied, which too is a polynomial-time operation.
\end{proof}

\begin{definition}[Out-Neighbour Consistency]
    We say that $v \in V$ is a vertex with out-neighbours consistent with respect to $A,B$ (where $A \cap B = \emptyset$ and $A,B \subseteq N_o(v)$) if there exists no pair of vertices $x \in A, y \in B$ such that $\phi_V(x) = \phi_V(y)$.
    \\
    We say that $S \subseteq V$ is a vertex set with out-neighbours consistent with respect to $A,B$ (where $A \cap B = \emptyset$ and $A,B \subseteq N_o(S)$) if there exists no pair of vertices $x \in A, y \in B$ such that $\phi_V(x) = \phi_V(y)$.
\end{definition}

\begin{definition}[Building Criteria]\label{definition:ComputeCriterion}
    Let \textsc{ComputeCriterion}{(B, M, E)} be any algorithm that takes three vertex sets $B,M,E$ as input and outputs a criterion such that every vertex in $B$ is classified as ``\textit{Yes}'', ``\textit{Included}'' or $1$, every vertex in $E$ is classified as ``\textit{No}'', ``\textit{Excluded}'' or $0$, and any vertex in $M$ but not in $B$ is classified as either.
\end{definition}

Such algorithms exist, with CART \cite{breiman1984charles} and C4.5 \cite{quinlan2014c4} being two notable examples.
In our case it is further important that when the tree pruning phase of these algorithms is initiated, pruning is done only if it does not break the guarantees of \Cref{definition:ComputeCriterion} or omitted altogether.

\begin{lemma}[Criterion existence for pseudo-bases with out-neighbours consistent]
    \label{lemma:voncExclusionGuaranteesCriterionExistence}
    Let $B,M,E$ be vertex sets such that $B \subseteq M,E \cap M = \emptyset$.
    If there exists a pseudo-basis $\mathcal{B}$ with out-neighbours consistent for $B,M$ then a criterion for $B,M,E$ as per \Cref{definition:ComputeCriterion} exists.
\end{lemma}

\begin{proof}
    Let $b \in B, e \in E$.
    Since $\mathcal{B}$ is a vertex set with out-neighbours consistent w.r.t $B,E$, $b,e$ cannot have the same features.
    Hence there exists a feature $f$ such that $\phi_V(b) \neq \phi_V(e)$.
    Thus, there exists a split $s_{b,e}$ on $f$ that separates $b,e$.
    A conjunction of these splits for all $b,e$ is a criterion for $B,M,E$ as per \Cref{definition:ComputeCriterion}.
\end{proof}

If a criterion exists, then both CART and C4.5, if left unterminated, will eventually build a decision tree that achieves the perfect separation.
Thus, either can be used as the \textsc{BuildCriteron} routine.

Our strategy to tackle STP mining is to find pseudo-bases with out-neighbours consistent for each step of a potential program, and then to find criteria (out-neighbourhood classifiers) that correspond to those pseudo-bases.
\Cref{lemma:voncExclusionGuaranteesCriterionExistence} shows that once an appropriate pseudo-basis has been found, the criterion can be found thanks to the consistency.

Multiple approaches can be taken to implement this strategy.
If getting \textit{a} solution is the priority, one can perform a depth-first search of pseudo-bases, and the moment the first valid sequence of pseudo-bases encapsulating $S$ at the beginning and hitting exactly $T$ at the end is found, find the step criteria and terminate.
Since by \Cref{lemma:voncExclusionGuaranteesCriterionExistence} we know that for pseudo-bases with out-neighbours consistent a criterion always exists, there is no utility in computing criteria on the while pseudo-bases are still being determined, as this process presents no benefit to the determination of pseudo-bases.
For the sake of consistency with earlier sections we chose to perform a breadth-first search of our pseudo-bases instead.

\medskip
\noindent In \Cref{alg:bpf}, let $Q, Q_{\text{next}}$ be queues lists of triples drawn from \textit{vertex-sets $\times$ vertex-sets $\times$ $\mathbb{Z}$}, $P$ be a set of \textit{programs}. The subject of our study is \Cref{alg:bpf}. Each list $\mathpzc{l}$ in $Q$ is represents a chain of pseudo-bases that might trace out a program path from $S$ to $T$. Every triple $(B, M, n)$ in $\mathpzc{l}$ and $Q_{\text{next}}$ represents from where to begin $B$ to reach $T$, the in-neighbourhood with out-neighbours consistent $B \subseteq M \subseteq N_i(N_o(B))$, and the saved distance $n$ from $B$ to $T$.

\vspace{0pt}

\begin{algorithm}[h!]
\caption{(BPF) Basis-Path-Finding}\label{alg:bpf}
\KwInit{$Q_{\text{next}}$ contains only the singleton list $(T, T, 0)$,\; $Q$, $\mathcal{L}$, $P$ are empty}

\For{$\ell \geq 1$ such that $T \subseteq N_o^{\ell}(S)$\label{line:bpfForwardPass}}{
    $Q \gets Q_{\text{next}}$ and empty $Q_{\text{next}}$\;
    \While{$Q$ is not empty\label{line:bpfBackwardPass}}{
        pop the list $\mathpzc{l}$ from $Q$\;
        let $(B, M, m)$ be the head, $n$ the length of $\mathpzc{l}$\;
        \;
        \If{$n-1=\ell$}{
            \If{$B \subseteq S \subseteq M$\label{line:bpfIfIsFeasible}}{
                add $\mathpzc{l}$ to $\mathcal{L}$\label{line:bpfAddFeasibleProgram}\;
            }
            push $\mathpzc{l}$ into $Q_{\text{next}}$\;
            \Continue
        }
        \;
        $M' \gets N_i(B) \cap N_o^{\ell-n}(S)$\;
        remove from $M'$ vertices with out-neighbours inconsistent w.r.t. $B,N_o(M') \backslash M$\;

        \ForEach{pseudo-basis $B' \subseteq M'$ for $(B, M)$\label{line:bpfSearchAllValidInjectionsStart}}{
            push $(B', M', n)$ into $\mathpzc{l}$\label{line:bpfAddExtendedPossibleProgram}
        }\label{line:bpfSearchAllValidMonochromaticPresetsEnd}
    }\label{line:bpfCombinatorialLoopEnd}
    \;
    \ForEach{list $\mathpzc{l}$ in $\mathcal{L}$\label{line:bpfEnumerateCandidateListsStart}}{
        $p \gets \epsilon$\;
        pop the head of $\mathpzc{l}$ and discard\;
        \ForEach{element $(B, M, n)$ in $\mathcal{L}$\label{line:bpfEnumerateCandidateListElementsStart}}{
            $c \gets$ \textsc{ComputeCriterion}{($B$, $M$, $N_o^{\ell-n}(S) \backslash M$)}\;
            $p \gets cp$\;
        }\label{line:bpfEnumerateCandidateListElementsEnd}
        add $p$ to $P$
    }\label{line:bpfEnumerateCandidateListsEnd}
}
\end{algorithm}

\begin{proposition}
     The following hold.\label{proposition:characteristicsOfBPF}
    \begin{enumerate}
        \item \Cref{alg:bpf} is correct in the sense that all programs in $P$ are exact programs for $S,T$.
        \item \Cref{alg:bpf} is complete in the sense that whenever execution exists the loop closing at \cref{line:bpfEnumerateCandidateListsEnd}, $P$ contains all exact programs for $S,T$ of length $\ell$, up to criterion equivalence.
    \end{enumerate}
\end{proposition}

\begin{proof}
    The proposition is analogous with the \Cref{proposition:characteristicsOfVIBE}, and as such analogous proofs can be constructed.
    
    \noindent Briefly, for correctness, if a candidate list 
    \[
        \mathpzc{l} = (B_0, M_0, \ell+1)(B_1, M_1, \ell) \cdots (B_{\ell}, M_{\ell}, 1)(T, T, 0)
    \]
    has been added $\mathcal{L}$, then it must have been the case that $S$ is a vertex set with out neighbours consistent spanning $B_1$ but not outspanning $M_1$.
    Inductively, it must have been the case that $B_k$ spans $B_{k+1}$ but does not outspan $M_{k+1}$.
    Now, the loop on \cref{line:bpfEnumerateCandidateListsStart} ensures that at each step, the program always proceeds to a set of vertices $S_k$ such that $B_k \subseteq S_k \subseteq M_k$.
    Hence any program in $P$ is correct.
    
    For completeness, just notice that following any $p \in P$ from the end there is always a pseudo-basis $B_k \subseteq E_{p_{\leq k}}(S)$ that spans $B_{k+1}$ but (trivially) does not outspan $E_{p_{\leq k+1}}(S)$.
    So a valid chain of pseudo-bases exists, will be found and added to $\mathcal{L}$, and the corresponding chain of criteria logically equivalent to those of $p$ (but not necessarily the same) will then be added to $P$.
\end{proof}

\vspace{-4pt}

As in the discussion of \Cref{alg:vibe}, \Cref{alg:bpf} can easily be modified to search for feasible and not just exact programs.

The runtime of the algorithms we propose depends greatly on the network. In the worst case -- on a fully connected graph, our algorithms perform a total enumeration of all possible programs. However, real-world labelled graph datasets tend to contain significant amounts of pattern structure, suggesting performance far better than that of the worst case.

\section{Conclusion}
\label{section:conclusion}
\vspace{-8pt}
We have pointed at the previously unaddressed problem of characterising relationships between groups of records in a database in terms of their long-distance connections within the database graph.
We have identified the problem of graph-walking program mining as a simple case of this wider challenge, and investigated simple colour and totally-ordered set program mining.
We addressed them by giving the Viable Injection Basis Enumeration and Basis-Path-Finding algorithms, and proved their correctness and completeness.

The main observation allowing us to sharply limit the search space is that the set of vertices through whom agent executing a simple program proceeds is bounded below by an appropriate basis and above by consistency with respect to out-vertices.
The construct corresponding to these bounds in the process of mining is the criterion \textit{pseudo-basis}, appearing as a necessary and sufficient condition on local existence of feasible programs.
We have further shown that the problems of simple program mining are $\NP$.

We have hinted that more complex program structures can be employed in graph-walking programs, and that the programs do not need to be deterministic.
While adding to the structure of graph-walking programs would likely hinder their interpretation as relationships between sets of vertices, we believe that our current work can be extended by considering probabilistic graph-walking programs, further expanding the utility of graph-walking programs in the context of network analysis and beyond.

\vspace{-4pt}

\bibliographystyle{splncs04}
\bibliography{bibliography}

\end{document}